\documentclass[letterpaper, 10 pt, journal, twoside]{IEEEtran}

\IEEEoverridecommandlockouts                              %
\makeatletter
\let\proof\@undefined
\let\endproof\@undefined
\makeatother
\IEEEoverridecommandlockouts
\usepackage{cite}
\usepackage{url}
\usepackage{hyperref}
\usepackage{amsmath,amssymb,amsfonts}
\usepackage{algorithmic}
\usepackage{graphicx}
\usepackage{textcomp}
\let\labelindent\relax
\usepackage{enumitem}
\usepackage{multirow}
\usepackage{booktabs}
\usepackage{xcolor}
\usepackage{flushend}

\usepackage{nicefrac}
\usepackage[font=small]{caption}
\usepackage{subcaption}
\usepackage{lipsum}

\usepackage[ruled,vlined,linesnumbered]{algorithm2e}
\usepackage{algorithmic,algorithm2e,float}

\SetKw{Continue}{continue}
\SetKw{Break}{break}

\newcommand{\vect}[1]{\boldsymbol{#1}}

\newcommand{\be}{\begin{equation}}
\newcommand{\ee}{\end{equation}}

\newcommand{\Xcal}{\mathcal{X}}
\newcommand{\Acal}{\mathcal{A}}
\newcommand{\Scal}{\mathcal{S}}

\newcommand{\w}{\vect{w}}

\newcommand{\T}{\mathcal{T}}

\newcommand{\f}{\vect{f}}

\newcommand{\clin}{c^{\mathtt{sum}}}
\newcommand{\ccheb}{c^{\mathtt{max}}}

\newcommand{\ie}{\textit{i.e.,}~}

\newcommand{\grey}[1]{\textcolor{gray}{#1}}

\usepackage{amsthm}
\theoremstyle{definition}
\newtheorem{problem}{Problem}

\newtheorem{definition}{Definition}
\newtheorem{proposition}{Proposition}

\newtheorem*{remark*}{Remark}

\newtheorem{lemma}{Lemma}

\SetKwInput{KwInput}{Input}                
\SetKwInput{KwOutput}{Output}              

\title{\LARGE \bf
Scalarizing Multi-Objective Robot Planning Problems \\ using Weighted Maximization}

\author{Nils Wilde, Stephen L.~Smith, and Javier Alonso-Mora
\thanks{This research is supported by the European Union's Horizon 2020 research and innovation program under Grant 101017008.}
\thanks{N.~Wilde and J.~Alonso-Mora are with the Department for Cognitive Robotics, 3ME,
Delft University of Technology, Delft, Netherlands, 
\texttt{\{n.wilde, j.alonsomora\}@tudelft.nl}.
S.~L.~Smith is with the Department for Electrical and Computer Engineering, University of Waterloo, Waterloo, Canada,
\texttt{stephen.smith@uwaterloo.ca}
}%
}

\begin{document}

\maketitle
\thispagestyle{empty}
\pagestyle{empty}

\begin{abstract}
When designing a motion planner for autonomous robots there are usually multiple objectives to be considered. However, a cost function that yields the desired trade-off between objectives is not easily obtainable. A common technique across many applications is to use a weighted sum of relevant objective functions and then carefully adapt the weights. However, this approach may not find all relevant trade-offs even in simple planning problems. Thus, we study an alternative method based on a weighted maximum of objectives. Such a cost function is more expressive than the weighted sum, and we show how it can be deployed in both continuous- and discrete-space motion planning problems.
We propose a novel path planning algorithm for the proposed cost function and establish its correctness, and present heuristic adaptations that yield a practical runtime.
In extensive simulation experiments, we demonstrate that the proposed cost function and algorithm are able to find a wider range of trade-offs between objectives (\textit{i.e.,} Pareto-optimal solutions)
for various planning problems, showcasing its advantages in practice. 
\end{abstract}

\section{Introduction}
\begin{figure}[ht]
    \centering
    \begin{subfigure}[!ht]{0.4\textwidth}
        \centering
        \includegraphics[width=1\textwidth]{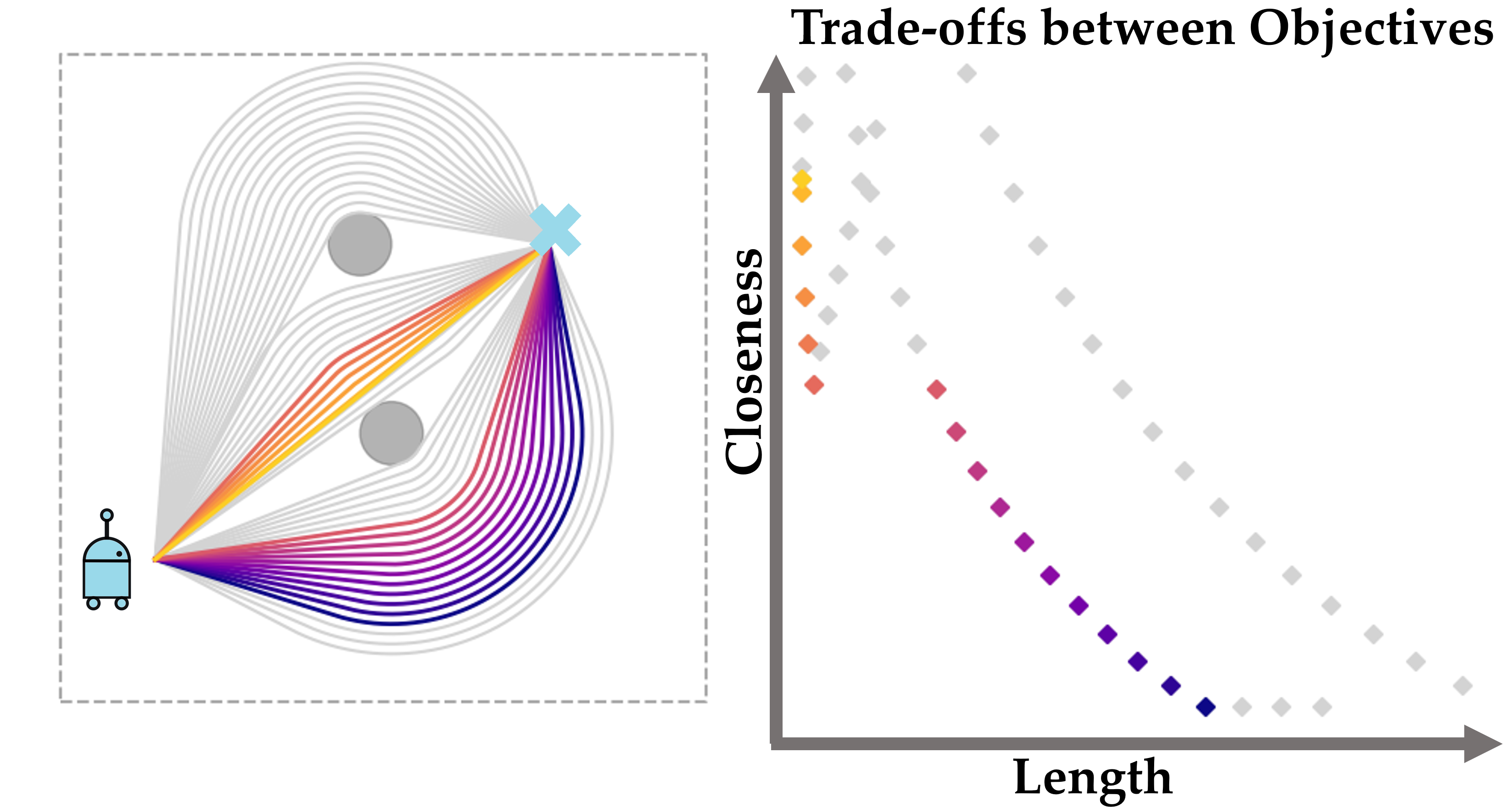}
        \caption{Ground set of feasible trajectories. Pareto-optimal solutions are highlighted in colour.}
        \label{fig:intro_all}
    \end{subfigure}
    \hfill
    \begin{subfigure}[t]{0.43\textwidth}
        \centering
        \includegraphics[width=1\textwidth]{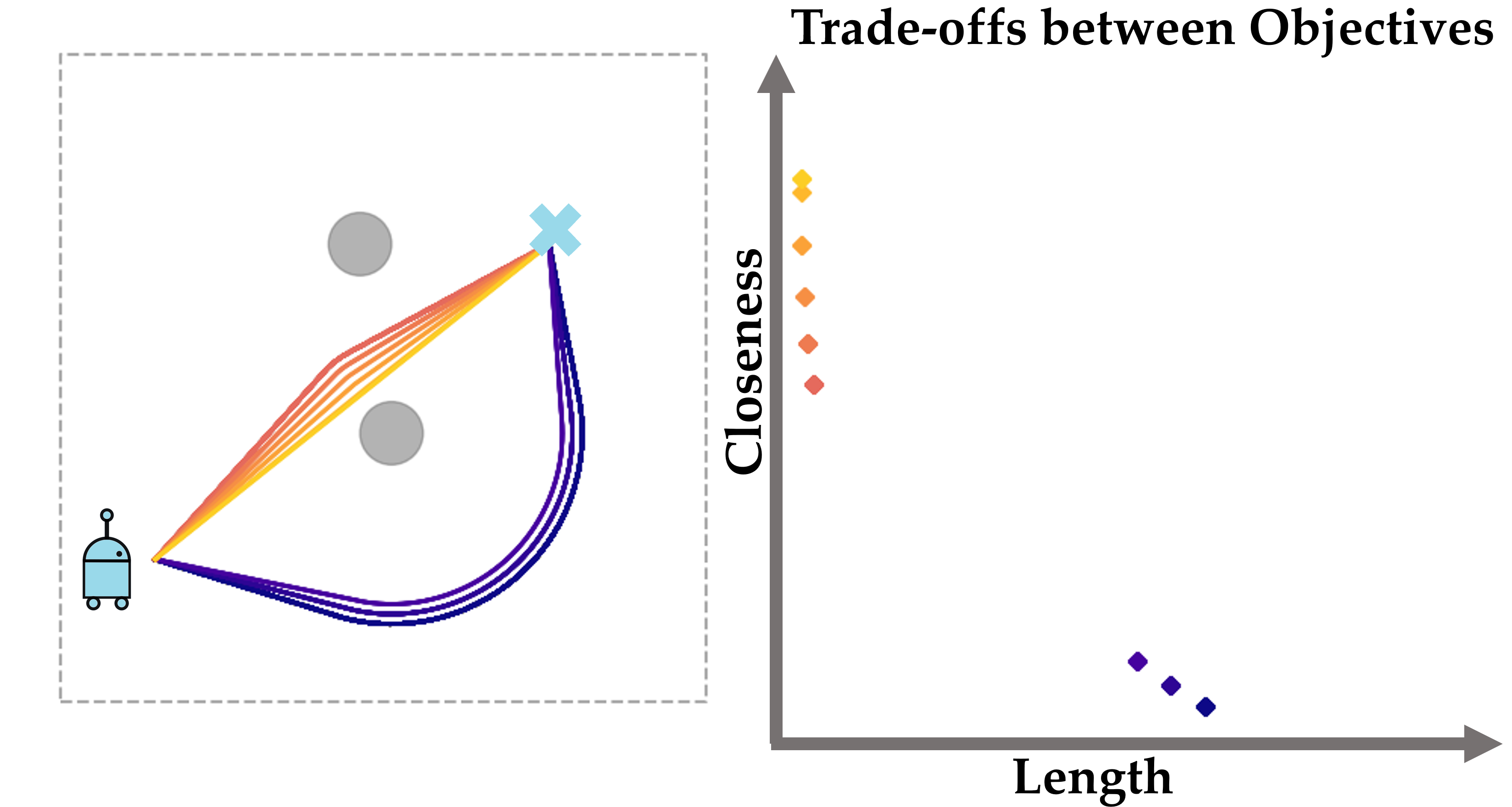}
        \caption{Attainable solutions for weighted sum (WS) optimization.}
        \label{fig:intro_lin}
    \end{subfigure}
    \hfill
    \begin{subfigure}[t]{0.43\textwidth}
        \centering
        \includegraphics[width=1\textwidth]{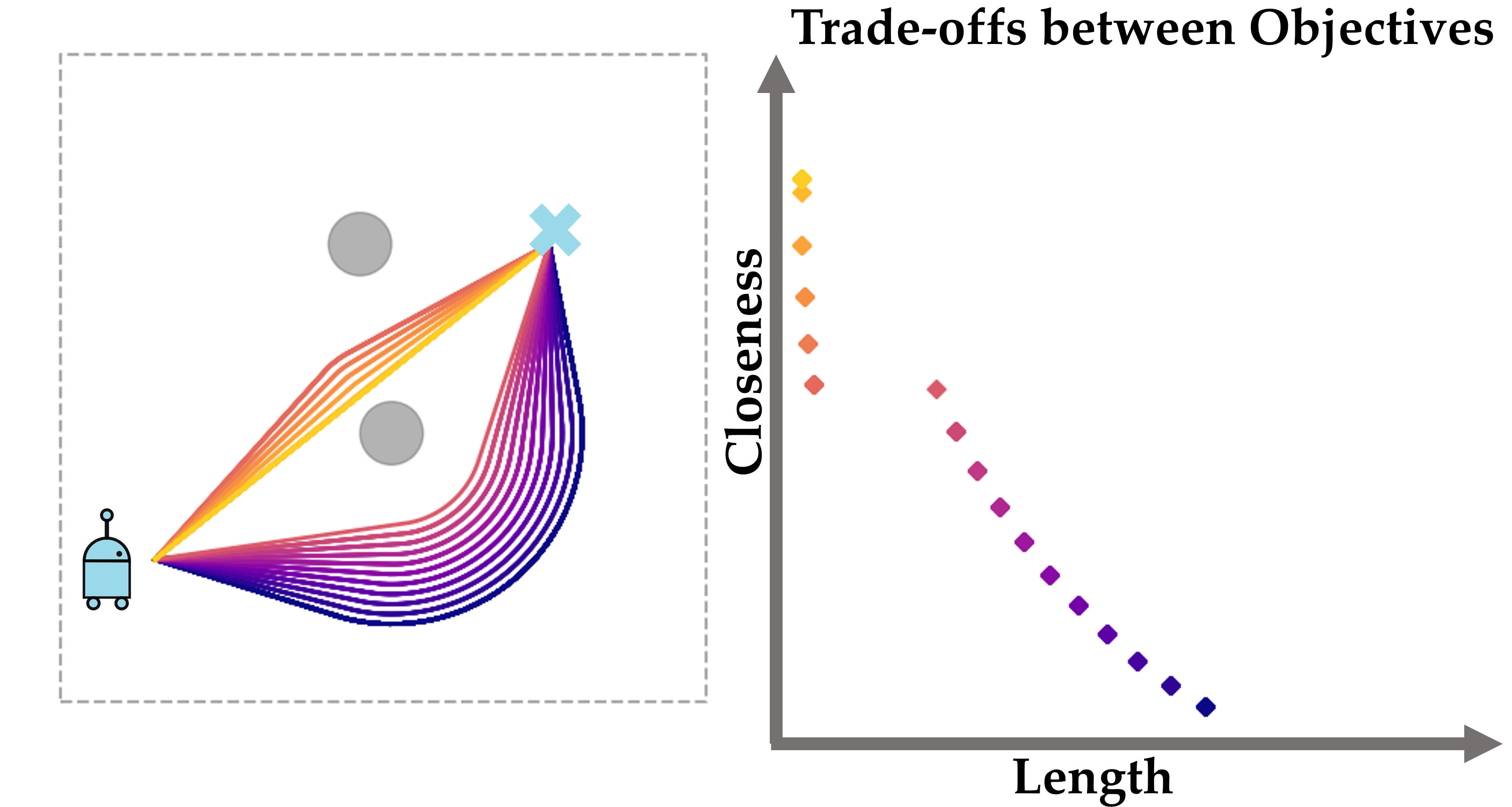}
        \caption{Attainable solutions for the proposed weighted maximum (WM) optimization.}
        \label{fig:intro_cheb}
    \end{subfigure}
    \caption{ 
    A simple planning problem with two objectives: trajectory length and the minimum distance to an obstacle. Shown are the ground set of trajectories (a), the solutions when using a weighted sum objective (b), and the solutions when using the proposed weighted maximum objective (c). In each subfigure the left plot shows the set of trajectories and the right plot shows the corresponding trade-offs. Colored trade-offs are Pareto-optimal, \textit{i.e.,} show the Pareto-front.}
    \label{fig:intro}
\end{figure}
Automated planning and decision making plays a central role in designing intelligent robotic systems.
In many real-world settings, autonomous robots are faced with complex scenarios that require them to balance between different objectives simultaneously. 
For instance, autonomous vehicles need to navigate to a goal, while ensuring safety, passenger comfort and ideally fuel-efficiency \cite{levinson2011towards,christianos2023planning, botros2022tunable}. Similarly, mobile robots navigating in human-centered spaces such as offices, hospitals or public areas need to consider task efficiency and conforming to social norms \cite{wilde2020improving, biswas2022socnavbench}.

In multi-objective optimization (MOO) problems -- such as finding trajectories trading off different objectives -- the optimal solution is usually not unique, but rather there is a set of \textit{Pareto-optimal} solutions. A solution is Pareto-optimal when none of the individual objectives can be improved without worsening at least one other objective. 
Thus, an important challenge in motion planning remains the design of objective functions that balance between several potentially competing objectives and allow for computing Pareto-optimal solutions. A common approach is to formulate a weighted sum of the objective functions \cite{wilde2020improving, zucker2013chomp, christianos2023planning,lu2020motion}.
Often, the weights on the objectives are tuning-parameters, requiring careful calibration. 
In human-robot interaction (HRI) user preferences for robot behaviour is commonly modelled as a weighted sum of features \cite{abbeel2004apprenticeship, hadfield2016cooperative, sadigh2017active, biyik2019asking, wilde2020improving,habibian2022here}.
The linear structure allows for designing efficient algorithms for both motion planning and learning from human feedback making this approach very popular.

However, it can fail to describe all optimal trade-offs  since it is unable to explore non-convex regions of Pareto-fronts. This issue persists even in simple planning problems, as illustrated in Figure \ref{fig:intro}. Here we compute trajectories between a fixed start and goal position around two static obstacles. To optimize for trajectory length and the minimal distance to an obstacle we can inflate the obstacles and then plan paths on visibility graphs.
Subplot \ref{fig:intro_all} shows the ground set of feasible trajectories together with the corresponding Pareto-front. Subplot \ref{fig:intro_lin} shows the trajectories that can be computed with the weighted sum (WS) method for different weights.
We observe that the trajectories found by the WS method are only a small subset of the Pareto-optimal solutions from \ref{fig:intro_all}:
While there are numerous trajectories available that pass between the two obstacles, there are only few that go around. 
It is important to notice that this is not due to the choice or resolution of the weights. Rather, \emph{there does not exist any tuning of weights} such that the motion planner returns a more intermediate trade-off, since parts of the Pareto-front are non-convex \cite{branke2008multiobjective}.
{
This problem occurs when parts of the Pareto-front are non-convex: the solutions of the weighted sum method do only cover the convex hull of the Pareto-front \cite{branke2008multiobjective}.
}

We study an alternative form of scalar objective function, the \textit{weighted maximum} (WM) of objectives, also known as Chebyshev scalarization \cite{branke2008multiobjective}. 
This allows for finding a richer set of trade-offs between the two objectives as shown in subplot \ref{fig:intro_cheb}, covering all parts of the Pareto-front.
Indeed, this approach is able to find all Pareto-optimal solutions, \textit{i.e.,} is \textit{Pareto-complete} \cite{branke2008multiobjective}.
Despite the theoretical foundations established in the optimization literature, more expressive scalarization methods such as the WM have not found much attention in robot motion planning. To demonstrate the potential of WM optimization in robot motion planning we discuss fundamental shortcomings of widely used WS cost functions, independent of how weights are selected. We revisit established results from optimization to describe theoretical differences between WS and WM cost functions: WM cost is a provably more expressive tool for motion planning, yet only requires the same number of parameters. 
{We show how WM costs can be used in continuous and discrete space planning problems. For discrete planning, we consider a general monotonic utility function to combine objective values of discrete actions (\textit{e.g.,} edges in graphs), allowing for complex planning.}

\subsection{Contributions}

Our contributions are as follows: 
{First, we consider continuous-space planning problems and show how existing optimization techniques can be used to solve planning problems with WM cost functions.
Further, show NP-hardness of graph based path planning with a WM cost.}
Second, we present a novel optimal path planning algorithm for the WM cost and establish its correctness. Further, we show how our algorithm can be enhanced with a cost-to-go heuristic and discuss a budgeted suboptimal version that runs in polynomial time.
Third, in a series of simulations, we demonstrate that the proposed WM method finds a substantially richer set of trade-offs in various motion planning problems, and showcase that the proposed graph search finds optimal solutions within a practical runtime.

\subsection{Related work}
Many robot planning problems consider multiple, potentially competing objectives. The most prominent approach is to formulate a weighted sum of a given set of objective functions and then solve the resulting scalar optimization problem. This approach is used in trajectory planning for autonomous driving \cite{tunable_traj_planner_urban, levinson2011towards, christianos2023planning, karkusdiffstack, botros2022tunable, zuo2020mpc, hang2020integrated}, local planning in cluttered environments \cite{lu2020motion,de2022risk} or social spaces \cite{che2020efficient, brito2019model}, trajectory generation for manipulators \cite{zucker2013chomp, marcucci2022motion}, and multi-robot planning \cite{luis2020online, cap2018multi}.

Researcher in HRI use weighted sums of objective functions -- usually referred to as features -- to model how users evaluate robot behaviour \cite{abbeel2004apprenticeship, wilde2021learning, hadfield2016cooperative, sadigh2017active, wilde2020active, biyik2019asking,  wilde2022learning,habibian2022here}. In human-in-the-loop learning frameworks users provide feedback to a robot in form of demonstrations, choice, labels, critic, and others, allowing the robot to learn weights for the objective functions and thus adapt their behaviour to the user's preferences.

Different approaches that address the shortcomings of WS objectives include alternative scalar functions such as the WM approach, or incremental algorithms that iteratively explore the entire Pareto-front \cite{branke2008multiobjective}.
A popular approach that can explore non-convex Pareto-fronts is the Adapt Weighted Sum Method \cite{kim2005adaptive, kim2006adaptive}. Here solutions are sampled iteratively using equality constraints to force new samples to close gaps in the Pareto-front. 
Yet, this approach is solving a slightly different problem than we are addressing: The method iteratively creates a set of potential solutions that cover the Pareto-front. Thus, it does not have tuning parameters such that carefully choosing them allows for obtaining the desirable solution. Further, the approach does not come with a completeness guarantee and can get stuck when the Pareto-front is discontinuous. Lastly, satisfying the added equality constraints can be infeasible or computationally hard in practice, especially in discrete-space planning. In our work, we propose using a WM cost function which has tunable weights such that any Pareto-optimal solution can be attained for a specific weight vector. Further, we explicitly address the challenges in discrete space planning and propose a novel graph search for minimizing WM costs for different types of individual objective functions.

Overall, the limitations of WS costs find less discussion in the context of robot planning.
Recently, \cite{thoma2023prioritizing} compared different scalar objective functions. Our work focuses on the WM cost function only, yet provides a theoretical analysis of its expressiveness and 
presents a novel algorithm for discrete space planning.
Other works directly address MOO for specific problems. For instance, the authors of \cite{yi2015morrf} studied weighted sum and weighted minimum approach for exploring Pareto-fronts of sampling based motion planning problems, while \cite{sakcak2021complete} addresses the problem of simultaneously optimizing for path length and clearance in the plane, proposing a complete and efficient algorithm. 
WM cost functions also found attention in multi-objective Reinforcement Learning (RL) \cite{van2013scalarized, chen2019meta}.
{
The works of \cite{ding2014hierarchical} and \cite{bopardikar2015multiobjective} study hierarchical frameworks based on a multi-objective Probabilistic Roadmap (MO-PRM) and explicitly consider two objectives: path length and risk \cite{ding2014hierarchical}, and path length and state-estimation error \cite{bopardikar2015multiobjective}. The MO-PRM separates objectives in primary and secondary costs, and then plans using a discretization of values for the secondary costs, similar to the $\epsilon$-constrained method \cite{branke2008multiobjective}. A drawback of that method is that it depends on the resolution of the constraint on the secondary cost and can require solving multiple optimization problems in order to verify Pareto-optimality.}
In contrast to these works, our paper does not address a specific multi-objective motion planning problem, but rather proposes an alternative to weighted sums for \textit{any} collection of objective functions.

\section{Problem Statement}
In this section, we revisit some preliminary concepts before introducing our formal problem statement.
\subsection{Preliminaries}
\paragraph{Pareto-optimality}
Consider a multi-objective optimization problem where the domain is some vector space $\mathcal{X}$. We want to find a solution $\vect{x}\in \mathcal{X}$ that simultaneously minimizes $n$ different functions, \ie  that solves $\min_{\vect{x}}\{f_1(\vect{x}), \dots, f_n(\vect{x})\}$. In general, the solution to a MOO problem is not a unique vector $\vect{x}$, but a set of \textit{Pareto-optimal} solutions. 
We briefly review the definitions of \textit{dominated solutions} and the Pareto-front.
\begin{definition}[Dominated solution]
    Given a MOO problem and two solutions $\vect{x},\vect{x}'\in\mathcal{X}$. Vector $\vect{x}$ \textit{dominates} $\vect{x}'$ when $f_i(\vect{x})\leq f_i(\vect{x}')$ holds for all $i=1,\dots,n$, and $f_j(\vect{x})< f_j(\vect{x}')$ holds for at least one $j$ where $1\leq j\leq n$. 
\end{definition}

\begin{definition}[Pareto-front]
    Given a MOO problem, the set of \textit{Pareto-optimal} solutions -- called \textit{Pareto-front} -- is the subset of all solution that are not \textit{dominated} by another solution.
\end{definition}

\paragraph{Graph theory}
Following \cite{korte2011combinatorial}, a graph is a tuple $G=(V,E)$ where $V$ are vertices and $E$ is a set of edges. In a weighted graph $G=(V,E,d)$ edges are associated with some cost $d:E\to \mathbb{R}$.
A walk is a sequence $v_1, e_1, v_2,\dots, v_k,e_k,v_{k+1}$ such that $e_i=(v_i,v_{i+1})\in E$ and $e_i\neq e_j$ for $i,j=1,\dots,k$. 
{We define a path $P$ as a sequence of vertices $(v_1, \dots, v_{k+1})$ with no duplicate entries for which there exists a walk $v_1, e_1, v_2,\dots, v_k,e_k,v_{k+1}$ in $G$.}

\subsection{Problem Formulation}

We consider a planning problem described by a robot's state and action space $(\Xcal, \Acal)$, a start state $x_s$ and a set of goal states $X_g\subset\Xcal$.
Let $\T$ be the set of all feasible trajectories starting at $x_s$ and ending at some state $x_g\in X_g$.
Note that the set $\T$ is typically defined implicitly as set of constraints on the robot's state and actions, such as kinodynamic constraints on motion, or spatial constraints for obstacle avoidance.  We keep this set abstract at this point, but give specific examples in Section~\ref{sec:results}.

To define the desired robot behaviour the designer of a motion planner considers a set objectives to be minimized. Let these objectivess be denoted by
$f_1, \dots, f_n$ where $f_i:\T\to \mathbb{R}_{\geq0}$ for $i=1,\dots,n$.
The optimal solution to the motion planning problem is some trajectory $T^*\in\T$. Assuming that the objectives $f_1, \dots, f_n$ contain all aspects under consideration, $T^*$ is a Pareto-optimal solution to the problem
\be
\min_{T\in \T}\{f_1(T), \dots, f_n(T)\}.
\label{eq:MOO problem}
\ee
Let $\T'\subseteq \T$ denote the set of all Pareto-optimal solutions.
Given above definitions, we can pose our main problem.
\begin{problem}[Parametric single objective planning]
\label{prob:main_prob}

Given state and action space $(\mathcal{X}, \mathcal{A})$, initial state $x_s$ and goal states $X_g$, and objectives $f_1, \dots, f_n$ 
{find an algorithm such that, for any Pareto-optimal solution $T^*\in \mathcal{T}'$, there exists algorithm parameters for which the algorithm returns $T^*$.}
\end{problem}

{Our approach to Problem \ref{prob:main_prob} is writing \eqref{eq:MOO problem} as a scalar function where tuning weights $\w$ define the balance between objectives.
The scalar function needs to be solvable, and for any Pareto-optimal trajectory $T^*\in\T$ there exist a choice of weights such that $T^*$ is the solution to the scalar optimization problem.
}

\section{Approach}

A common approach to tackle Problem \ref{prob:main_prob} is solving the MOO problem from \eqref{eq:MOO problem} via means of linear scalarization, also referred to as the weighted sum (WS) method or cost \cite{branke2008multiobjective}. 
This yields the following cost function
\be
\clin(T) = \sum_{i=1}^n f_i(T) w_i = \f(T)\cdot \w,
\label{eq:lin_cost}
\ee
where $\w\in[0,1]^n$ is a vector of tunable weights.
While this approach has been widely used  and been proven to be effective, its simplicity limits the expressiveness.
In this paper, we offer an alternative model-based approach.
We propose a weighted maximum approach for a scalar cost functions, where the summation is replaced by taking the maximum:
\be
c'(T) = \max_{i=1,\dots,n} f_i(T) w_i.
\label{eq:cheb_cost_raw}
\ee

The cost of a trajectory is now given by the objective that attains the largest value when multiplied by its weight.
That is, a trajectory is evaluated only based on the most prominent weighted objective value, and disregards other objectives.
We notice that when the solution to~\eqref{eq:cheb_cost_raw} is unique, it is Pareto-optimal.  However, if there are multiple solutions, then one is Pareto-optimal, while all others are only \textit{weakly} Pareto-optimal~\cite{branke2008multiobjective}.
In order to only attain Pareto-optimal solutions, we add $\rho\sum_{i=1}^n f_i(T)$ as a tie-break in the cost function, where $\rho> 0$ is a sufficiently small constant:
\be
\ccheb(T) = \max_{i=1,\dots,n} f_i(T) w_i + \rho \sum_{i=1}^n f_i(T).
\label{eq:cheb_cost}
\ee

We refer to this as the weighted maximum (WM), or \textit{augmented Chebyshev problem} \cite{branke2008multiobjective}.
Next, we characterize its expressiveness compared to the WS.
Given a planning problem with the ground set of feasible trajectories $\T$, let $\T'\subseteq \T$ be the set of all Pareto-optimal trajectories. Further, let $\T^{\mathtt{sum}} \subseteq\T$ be the set of trajectories that are optimal for \textit{some} weight in \eqref{eq:lin_cost}
and let $\T^{\mathtt{max}} \subseteq\T$ be the set of trajectories that are optimal for \textit{some} weights in \eqref{eq:cheb_cost}.
{
In detail, we have $\T^{\mathtt{sum}} = \{T' \in \T\ | \ T' = \arg\min_T c^{\mathtt{sum}}(T),\ \w\in[0,1]^n\}$, and $\T^{\mathtt{max}} = \{T' \in \T\ | \ T' = \arg\min_T c^{\mathtt{max}}(T),\ \w\in[0,1]^n\}$.
}
We revisit a known result:

\begin{lemma}[Pareto-optimality of scalarization]
For any planning problem $\T^{\mathtt{sum}} \subseteq\T^{\mathtt{max}} \subseteq\T'$.
\end{lemma}

A proof is omitted since this is a well established result in multi-objective optimization \cite{branke2008multiobjective}. The lemma ensures that any solution to the two scalarized optimization methods is always a Pareto-optimal solution. 
However, a lesser known result is while all solutions to \eqref{eq:lin_cost} are Pareto-optimal, there can exist Pareto-optimal solutions that are not a solution to \eqref{eq:lin_cost} for any $\w$.
Thus, the WS is less expressive than the WM.

\begin{proposition}[Expressiveness]
\label{prop:expressive}
Given a planning problem where trajectories $T$ are parametrized in $\mathbb{R}^m$, and auxiliary cost functions $\f(T)=[f_1(T)\, \dots\, f_n(T)]$.
If at least one of the cost functions $f_i(T)$ is not a proper convex\footnote{{A convex function $f(x)$ is proper convex over its domain $\mathcal{X}$ if it never attains $-\infty$, and if there exists at least some $x_0$ where $f(x_0)<\infty$.}} and continuous function over $\mathbb{R}^m$ then $\T^{\mathtt{sum}}\subset \T^{\mathtt{max}}$, \textit{i.e.,} optimizing \eqref{eq:lin_cost} is strictly less expressive than optimizing \eqref{eq:cheb_cost}.
\end{proposition}

The proof follows directly from Theorem 6.3 and Remark 6.4 in \cite{censor1977pareto}: Linear scalarization is only Pareto-complete when the costs are proper convex and continuous. Thus, when that condition is violated we have $\T^{\mathtt{sum}}\subset \T'$. In contrast, the WM is Pareto-complete \cite{branke2008multiobjective}, \textit{i.e.,} $\T^{\mathtt{max}}=  \T'$ always holds.

The effect of Proposition \ref{prop:expressive} becomes apparent in Figure \ref{fig:intro}. 
When there is more than one homotopy class for navigating around obstacles, the objective for minimizing closeness to obstacles becomes non-convex. As a consequence, the Pareto-front is non-convex. The WS method is then only able to find solutions lying on the convex hull of the Pareto-front, while the WM method finds solutions in all parts of the Pareto-front.

\section{Motion planning with weighted maximum cost}

We now consider the problem of finding an optimal trajectory for the proposed WM cost for given weights $\w$.
Thus, we study how the WM cost can be used in continuous space motion planners such as Model-Predictive Control (MPC), and in discrete, graph-based planners such as state-lattices.

\subsection{Continuous space planning}
{We consider a discrete time, continuous space planning problem to find an optimal trajectory $T$, subject to kinodynamical constraints $g(T) \leq 0$. When minimizing the proposed WM cost, the problem may be written as}
\be
\begin{aligned}
\label{eq:continuous}
\min_{T} &\,\max_i f_i(T) w_i + \rho\sum_{i=1}^n f_i(T)\\
s.t.&\, g(T)\leq 0.
\end{aligned}
\ee
Following \cite{boyd2004convex}, this can be reformulated {as follows}
\be
\label{eq:continuous_reformulated}
\begin{aligned}
    \min_T &\;t + \rho\sum_{i=1}^n f_i(T)\\
    s.t.&\,\max_i\ w_i f_i(T)  \leq t\; \text{for } i=1,\dots,n,\\
    &\,g(T)\leq0.
\end{aligned}
\ee
The newly introduced max constraint can be written out as $n$ constraints of the form $w_i f_i(T)  \leq t$ for $i=1,\dots, n$.
{We observe that this removes the maximization from the problem, such that the objective and constraints are linear compositions of the individual objective $f_i$.}
In case constraints $g(T)$ are non-convex, solving for the weighted maximum does not make the problem fundamentally harder than optimizing for the weighted sum.
However, the same does not hold for graph-based planners as we will show next.

\subsection{Graph based planning}
We now consider the WM cost for discrete space motion planners such as graph or lattice based methods and characterize the hardness of the problem. 
Let $G=(V,E)$ be a graph where we associate each edge $e\in E$ with non-negative and bounded trajectory costs $f'_1(e), \dots, f'_n(e)$.\\

\paragraph{LP formulation}
We briefly consider the simple case where the costs of a path are the sum of the edge costs  $f_i(P)=\sum_{e\in E(P)} f_i(e)$. We recall the linear program (LP) formulation of a shortest path problem, \textit{i.e.,} a path minimizing \eqref{eq:lin_cost}. Thus, the cost of an edge $e$ is given by $\f(e) \cdot \w$, and the network flow constraints are summarized as $F(\vect{x})\leq\vect{b}$. The well-known LP-formulation \cite{bertsimas1997introduction} then is
 \be
\begin{aligned}
    \min_{\vect{x}} &\; \sum_{(v,u)} x_{vu}\cdot\f(e_{vu}) \cdot \w\\
    s.t.&\, F(\vect{x})\leq \vect{b}.
\end{aligned}
\label{eq:shortest_path_LP}
\ee
Here $\vect{x}$ is a binary vector, where $x_{vu}=1$ indicates that edge $(v,u)$ is contained in the path. 
When now considering the WM cost, the objective is $\min_{\vect{x}} \;\max_i \  w_i \sum_{(v,u)} x_{vu}\cdot f'_i(e_{vu})+ \rho\sum_{i=1}^n f_i(e_{vu})$. In principle, we can apply the same re-formulation technique as in equation \eqref{eq:continuous_reformulated} and still obtain an LP. However, the solution will not have integer values since the constraints are no longer totally-unimodular. Hence, the LP solution will not solve the shortest path problem \cite{bertsimas1997introduction}.

\paragraph{Formal problem analysis}
Given that the LP-formulation for shortest paths does not work for the WM cost, we study the problem of finding a path that minimizes \eqref{eq:cheb_cost} in more detail.
We consider the general case where the costs of a path $P$ are not necessarily the sum of the edge costs in the path. Instead, the cost for a path $P$ with edges $E(P)$ is
\be
\label{eq:path_costs}
f_i(P)=\beta(f'_i(e_1), f'_i(e_2), \dots), \text{ where } e_1, e_2, \dots \in E(P).
\ee
We refer to $\beta$ as the composition function and assume that $\beta$ is monotonically increasing. This captures two widely used concepts of defining costs over a robot's trajectory: i) summation or integration over the trajectory to compute its length, time, integral square jerk, accumulated risk or similar costs, and ii) taking the maximum value over a trajectory such as the maximum jerk or maximum risk.
Thus, we can state the problem of finding a path of minimal maximum weighted cost.
\begin{problem}[Min-max cost path (MMCP)]
\label{prob:mincostPath}
    Given a strongly connected graph $G=(V,E)$ with start and goal vertices $s,g$ in $V$, edge cost functions $f'_1(e), \dots, f'_n(e)$, a composition function $\beta$, and weights $w_1, \dots, w_n$, find a path that solves 
    \be
    \min_P \max_i w_i f_i(P)+ \rho\sum_{i=1}^n f_i(P).
    \ee
\end{problem}

The problem is closely related to the multi-objective shortest path (MOSP) problem{, which is NP-hard for two or more objectives.} \cite{ehrgott2005multicriteria}. The main difference is that MOSP considers that $\beta$ is taking the sum over different edge cost, which makes it a special case of Problem \ref{prob:mincostPath}. 
We formally establish hardness of our problem.
\begin{proposition}[Hardness of MMCP]
    The MMCP is NP-hard.
\end{proposition}

\begin{proof}
    We consider the special case that $f_i(P) = \sum_{e\in E(P)} f_i(e)$. 
    The decision version of MMCP decides if there exists a path such that the $\max_i w_i f_i(P) \leq \alpha$ for some constant $\alpha$.
    We can reduce MOSP to MMCP. MOSP takes as an input a graph $G=(V,E)$ and some cost functions $\gamma_1,\dots, \gamma_n$ assigning costs to edges. The decision version answers if there exists a path $P$ such that $\sum_{e\in E(P)}\gamma_i(e)\leq \alpha$ for all $i$ and some constant $\alpha$.
    Given an instance of MOSP, we use the same graph as an input for MMCP, choose costs $f_i(e)=\gamma_i(e)$, and set $w_i=1$ for $i=1,\dots,n$.
    The solution to MMCP then indicates if $\max_i\sum_{e\in E(P)}\gamma_i(e)\leq \alpha$, which trivially also decides the MOSP instance.
\end{proof}
\bigskip

\paragraph{Algorithm description}
We now present a complete algorithm for MMCP, detailed in Algorithm \ref{alg:graph search}.
Our approach is a modification to Dijkstra's algorithm where we record all paths to a vertex, similar to the \textit{Martin algorithm} for MOSP \cite{gandibleux2006martins}.
{To that end, the elements in our $\mathtt{open\_list}$ are tuples consisting of a cost, a vertex, and a path (\textit{i.e.,} a sequence of preceding vertices) from the start to this vertex.
Similar to Dijkstra's our algorithm retrieves the lowest cost element from the $\mathtt{open\_list}$ (line 3). We then expand the neighbouring vertices $u$ (line 7) and ensure that the path to the neighbour is not in the $\mathtt{open\_list}$, is not dominated by another path to $u$, and does not contain cycles (lines 8-10). We then add the path to $u$ with its WM cost to the $\mathtt{open\_list}$ (line 12-13).
}
Algorithm \ref{alg:graph search} is able to handle any monotonically increasing composition function $\beta$ (see equation \eqref{eq:path_costs}) as opposed to the sum of edge costs considered in MOSP.
Opposed to MOSP, we are only interested in finding one solution for a given weight instead of the set of all Pareto-optimal paths. Thus, Algorithm \ref{alg:graph search} terminates once the goal is reached {(lines 4-5)}.\\

\begin{algorithm}[!t]
\DontPrintSemicolon
  
  \KwInput{Graph $G=(V,E)$, start and goal $s,g$ in $V$, cost functions on edges $\f=[f_1, \dots, f_n]$, composition function $\beta$, weight $\w$}
  \KwOutput{Path from $s$ to $g$ with minimal WM cost}

  $\mathtt{open\_list} = \{(0, s, (s)) \}$  \\
  
  \While{$|\mathtt{open\_list}|>0$}
  {
  $(cost, v, P)\leftarrow \mathtt{open\_list}.\mathtt{pop()}$ \grey{// get by min cost}\\
  \If{$v = g$}{
   \Return{$P$}
  }
  \For{$u$ in $\mathtt{Neighbours}(v)$}{
      Create tentative path $P^u = P\cup \{u\}$\\
      $\mathcal{P}\leftarrow$ all paths in $\mathtt{open\_list}$ that end at $u$ \\
      \If{Any $P'\in \mathcal{P}$ \textit{dominates } $P^u$ \textbf{or} $P^u$ contains cycles}{\Continue}
      Delete tuples from  $\mathtt{open\_list}$ where $P'$ is dominated by $P^u$\\
      Comp.~WM cost $\ccheb(P^u, \w)$ \grey{// Using  \eqref{eq:cheb_cost}, \eqref{eq:path_costs}}\\
        
      Add $(\ccheb(P^u, \w), u, P^u )$ to $\mathtt{open\_list}$
  
      }
    }
\Return{Failure}
\caption{Min-Max Cost Path}
\label{alg:graph search}
\end{algorithm}

\paragraph{Theoretical properties}

{First, we characterize the runtime.
In the worst case, Algorithm \ref{alg:graph search} explores all paths from $s$ to any vertex $u$, leading to $2^{|V|}$ (the size of the power set for all sequences of vertices) executions of the while loop. For each subpath, we compute its cost only once in line 12, which requires evaluating $f_i(e)$ for all its edges and all $n$ objective functions. The number of edges is upper bound by $|V|^2$. Assuming that the evaluation of the costs $f_i(e)$ takes constant time, the total runtime is $O(2^{|V|}\cdot |V|^2 \cdot n$). While the runtime only grows linearly with the number of objective functions, it can scale exponentially with the number of vertices. 
However, due to the stopping criteria in line 4, the algorithm does not enumerate all $2^{|V|}$ solutions in practice. In our simulations we show that it is able to solve instances with $|V|=2000$.
}

{Next we will establish correctness of the algorithm. We begin by considering the subpath elimination in line 9.}
\begin{lemma}[Subpath elimination]
\label{lem:subpath}
Let $P'$ and $P^u$ be two subpaths from $s$ to $u$ such that $P'$ dominates $P^u$. 
{If the composition function $\beta$ is monotone,} then $P^u$ cannot be part of an optimal path from $s$ to the goal $g$.
\end{lemma}
\begin{proof}
    If $P'$ dominates $P^u$ then $f_i(P')\leq f_i(P^u)$ for all $i$. Now consider any possible path $Q$ from $u$ to the goal $g$. Since $f_i$ is \textit{monotone} we then have $f_i(P'\cup Q)\leq f_i(P^u\cup Q)$. Finally, $c(P'\cup Q)\leq c(P^u\cup Q)$ follows directly from \eqref{eq:cheb_cost}.
    Hence, $P^u$ cannot lead to a path to the goal of lower cost than $P'$ and thus may be disregarded from the search.
\end{proof}
The result is similar to the analysis in \cite{ehrgott2005multicriteria}, yet extends to the case of any monotone composition function $\beta$ instead of only sums.
Based on Lemma \ref{lem:subpath}, we can ensure that our algorithm finds the optimal solution.
\begin{proposition}[Correctness]
\label{prop:correct}
For any given weight $\w$, Algorithm \ref{alg:graph search} returns the optimal solution $P^* = \arg\min_P\max_i f_i(P) \cdot w_i$.
\end{proposition}
\begin{proof}
    The proof follows three steps:
    (i) Lemma \ref{lem:subpath} ensures that we never eliminate the optimal path during the search.
    (ii) Eventually, a tuple $(cost, v, P)$ where $v=g$ will be pulled from the open list and we thus find a path to the goal. (iii) the first time such a tuple is retrieved from the open list, it must have the minimal cost in the open list, and since it is the element of minimal cost in the open list and the cost is monotone, there cannot be another subpath in the open list that, when extended until $g$, achieves a smaller cost.
\end{proof}

\bigskip
\paragraph{Cost-to-go heuristic}
{While Algorithm \ref{alg:graph search} is optimal, its runtime scales exponentially with the size of the graph.}
The runtime can be improved with a cost-to-go heuristic as in an \textsc{A$^*$} or \textsc{D$^*$} algorithms {\cite{koenig2004lifelong}}. To use a heuristic, we augment the path $P^u$ with a virtual edge to the goal. 
{This virtual edge allows for including an estimate for the cost-to-go. An \textsc{A$^*$} algorithm simply adds a heuristic value for the cost-to-go to the current cost. In contrast, our problem considers a maximization in the cost as well as a potentially non-linear composition function. Thus, we explicitly add an edge and then calculate the WM cost in line 12 for the augmented path.}
The objective values of the virtual edge must be chosen such that the WM cost of the augmented path is an underestimate of the optimal path to the goal. 
For instance, if one objective is length, we can set the length of virtual edge to the Euclidean distance while other objective values are zero.\\

\paragraph{Runtime Budgeting}
\label{sec:runtime_budget}
Finally, we can modify Algorithm \ref{alg:graph search} to find potentially suboptimal solutions in polynomial runtime{, similar to anytime algorothms such as ARA* \cite{likhachev2003ara}.} To that end, we introduce a budget $b$ for the number of predecessor paths leading to every vertex that we can store. We then only add a new tuple in line 13 when the number of tuples {with a path ending at} $u$ in the open list is below $b$.
{This prevents the $\mathtt{open\_list}$ to grow exponentially, yet might prevent the algorithm from finding an optimal solution.}

In summary, we have shown how the WM cost can be incorporated in continuous- and discrete-space planning problems. For graph based planning we provided hardness results together with a complete algorithm.

\section{Numerical Results}
\label{sec:results}
To illustrate the advantages of the WM method, we consider several motion planning problems with multiple objectives and compare the attainable solutions when using either WS and WM. Further, we investigate the runtime of Algorithm \ref{alg:graph search}.

\subsection{Comparison of WS and WM cost functions}

For given objective functions, we compare how expressive WM ans WS approaches are.
We approximate the sets of attainable solutions $\T^{\mathtt{sum}}$ and $\T^{\mathtt{max}}$ for the WS and WM cost functions as follows: 
We randomly sample a large set of weights $\w^1, \dots, \w^k$ and then compute the respective sets of optimal solutions $\Scal^{\mathtt{sum}}=\{T^{\mathtt{sum}}(\w^1), \dots, T^{\mathtt{sum}}(\w^k)\}$ solving \eqref{eq:lin_cost} and $\Scal^{\mathtt{max}}=\{T^{\mathtt{max}}(\w^1), \dots, T^{\mathtt{max}}(\w^k)\}$ solving \eqref{eq:cheb_cost}. 
Thus, $\Scal^{\mathtt{sum}}$ and $\Scal^{\mathtt{max}}$ are both subsets of the Pareto-front of the MOO described by the given objective functions, \textit{i.e.,} consist of Pareto-optimal solutions.

We use three quantitative measures to compare $\Scal^{\mathtt{sum}}$ and $\Scal^{\mathtt{max}}$: {dispersion}, {coverage} and {number of unique solutions} on the Pareto-fronts. 
The dispersion captures \textit{gaps} in the approximation of the Pareto-front, and is defined as follows:
\begin{definition}[Dispersion]
    Given solutions $\Scal=\{T^1, \dots, T^k\}$, the dispersion of $\Scal$ is the maximum distance between a point $\vect{p}$ on the Pareto-front and the closest $\f(T^i)$ for $i=1,\dots, k$.
\end{definition}
In principle, this distance should be defined as a measure along the Pareto-front, yet for practical purposes we use the Euclidean distance.
Note that for a useful interpretation of  dispersion measure, we require objectives to be normalized.
Coverage captures the volume of the set of points that is dominated by the solutions $\Scal$
\cite{zitzler1999multiobjective}. In a minimization problem with normalized features we sample over the set $[0,1]^n$ to estimate coverage.
Lastly, sampling $k$ different weights does often not lead to $k$ different solutions. Thus, given a set $\Scal=\{T^1, \dots, T^k\}$, we compute the number of trajectories $T^i$ where the Euclidean distance between $\f(T^i)$ and $\f(T^j)$ is above some threshold $\delta$ for all $i, j=1,\dots,k$. 
We refer to this measure as the number of unique solutions.
All experiments run with $k=200$ samples and distance threshold $\delta=0.01$.\\

\paragraph{Simple Obstacles}

First, we revisit the example from Figure \ref{fig:intro} and provide numerical results in Table \ref{tab:exp1} (labelled \textit{Obstacles}).
We observe that WM outperforms WS on all three metrics, and with a large margin  on dispersion on and number of solutions. This highlights the significant shortcomings of the WS even in very simple planning problems.\\

\paragraph{Continuous space motion planning}

The second experiments considers a continuous space motion planner. We use the driver experiment that is popular in numerous studies on reward learning in HRI, for instance \cite{sadigh2017active, biyik2019asking, wilde2020active}.
An autonomous car navigates on a three lane road in the presence of a human-driven vehicle. The problem considers four objectives: heading, position in the lane, speed and distance to the other car. We solve the problem numerically using a numerical solver for constrained non-convex optimization. The min-max objective is implemented as in equation \eqref{eq:continuous_reformulated}.

Qualitative results are shown in Figure \ref{fig:eval_driver}. Since the numerical solver may return suboptimal solutions, we filtered all trajectories that were dominated by another trajectory. Overall, the WM yields a larger variety of solutions. In particular, the solutions for the WS method are only variations of few types of trajectories, while WM offers more nuanced solutions.
On the evaluation measures, the WM clearly outperforms the WS with respect to dispersion and the number of unique solutions, yet by a smaller margin than in the \textit{Obstacles} experiment. For coverage WM has only a small benefit.\\
\begin{figure}[t]
    \centering
    \begin{subfigure}[t]{0.49\textwidth}
        \centering
        \includegraphics[width=1\textwidth]{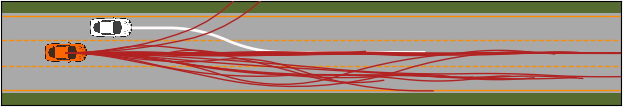}
        \caption{Weighted sum (WS).}
        \label{fig:eval_drver_lin}
    \end{subfigure}
    \\
    \begin{subfigure}[t]{0.49\textwidth}
        \centering
        \includegraphics[width=1\textwidth]{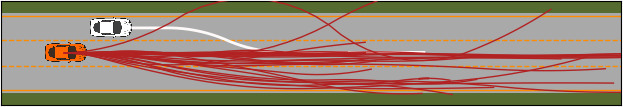}
        \caption{Weighted maximum (WM).}
        \label{fig:eval_driver_cheb}
    \end{subfigure}
    \caption{ 
     Results for the driver experiments. White shows the human driven car, red trajectories show solutions for the autonomous car.}
    \label{fig:eval_driver}
\end{figure}

\paragraph{Graph-based motion planning}

In the third setup we consider a probabilistic roadmap (PRM) with $1000$ vertices, shown in Figure \ref{fig:eval_graph}. Similar to the first experiment, the objectives are path length and closeness to obstacles. We consider two problem variations for closeness: the summed closeness, labelled as Graph-1 in Table \ref{tab:exp1} with an example shown in Figure \ref{fig:eval_graph}, and minimum closeness, labelled as Graph-2.

In Figure \ref{fig:eval_graph} we observe that the WM finds a larger variety of paths, some falling into a homotopy class for which the WS method does not find any path. In the Pareto-fronts WS exhibits several large gaps, while the WM covers the Pareto-front more densely. The gaps of the WS correspond to non-convex parts of the Pareto-front, implying that these parts cannot be covered by the WS for any choice of weights.
The measures in Table \ref{tab:exp1} show again a substantially smaller dispersion, slightly higher coverage and higher number of solutions for WM compared to WS in both graph problems.

\begin{figure}[t]
    \centering
    \begin{subfigure}[t]{0.2\textwidth}
        \centering
        \includegraphics[width=.9\textwidth]{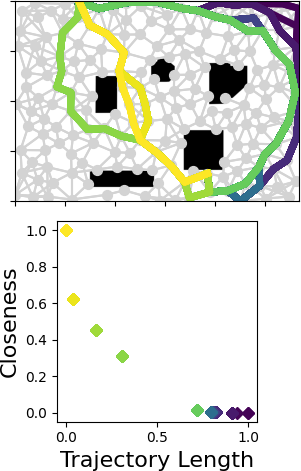}
        \caption{Weighted sum (WS).}
        \label{fig:eval_graph_lin}
    \end{subfigure}
\hfill
    \begin{subfigure}[t]{0.2\textwidth}
        \centering
        \includegraphics[width=.9\textwidth]{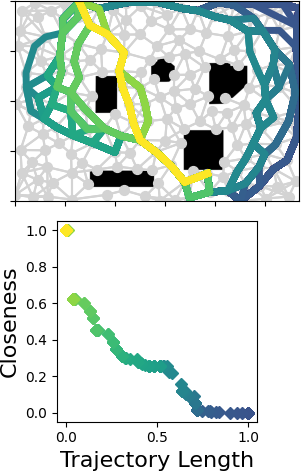}
        \caption{Weighted maximum (WM).}
        \label{fig:eval_graph_cheb}
    \end{subfigure}
    \caption{ 
     Optimal paths for a graph based motion planning problem with two objectives: Path length, and the summed distance to an obstacle. Upper: colored trajectories are solutions found by the weighted sum (a) and weighted maximum (b), respectively. Lower: Objective values of the trajectories of the upper plot.}
    \label{fig:eval_graph}
\end{figure}

\begin{table}[t]
    \centering
    \begin{tabular}{lr ccc}
        \toprule

        Problem  & Cost & $\downarrow$ Dispersion&  $\uparrow$ Coverage & $\uparrow$ $\#$ Solutions
        \\
        \midrule
        Obstacles        
         & $\mathtt{SUM}$  & 0.49 & 0.49 & 9 \\
        & $\mathtt{MAX}$ & 0.15 & 0.65 & 17\\       
         Driver 
         & $\mathtt{SUM}$ & 0.52 & 0.95 & 22 \\
    & $\mathtt{MAX}$ & 0.28 & 0.97 & 29   \\ 
        Graph-1
        & $\mathtt{SUM}$ & 0.26 & 0.71 & 11\\
        & $\mathtt{MAX}$ & 0.19 & 0.75 & 40 \\
        Graph-2
        & $\mathtt{SUM}$ & 0.42 & 0.73 & 6\\
        & $\mathtt{MAX}$ & 0.20 & 0.80 & 10 \\
        \bottomrule
    \end{tabular}
    \caption{Numerical results for different planning problems.}
    \label{tab:exp1}
\end{table}

In summary, in all three planning problems the proposed method is able to find better sets of Pareto-optimal trade-offs compared to the weighted sum method.

\subsection{Performance of WM planning on graphs}
\label{sec:experiment_runtime}

In a second experiment, we investigate the performance of Algorithm \ref{alg:graph search} when using the computation budget and the cost-to-go heuristic. As a heuristic, we add a virtual edge where the length equals the Euclidean distance to the goal and the closeness is zero.
We use a PRM with $2000$ vertices, similar to the one in Figure \ref{fig:eval_graph}. In $1000$ trials, we randomise start and goal locations, as well as the weights in the cost function.
Figure \ref{fig:eval_alg_performance} shows the cost ratio compared to optimal and the computation time for various computation budgets $b$.
For the standard implementation without heuristic, we observe that for $b=50$, all returned solutions are almost optimal (ratio $<1.001$). This comes with an increase in computation time by a factor of $250$ on average, but still remains below $3$ seconds {(Hardware specification: Intel i7-11800H @2.3GHz, 32Gb RAM.)}. Using the cost-to-go heuristic keeps average the runtime increase below a factor of $35$ (or $<.5$ seconds). Moreover, the heuristic also allows for finding close-to-optimal solutions with a very small budget of $b=5$, where the runtime increase is negligible.
Only for $b=1$ the heuristic can misguide the search and yield suboptimal solutions. Yet, this does not invalidate the admissibility of the chosen heuristic: For small $b$ the algorithm has no guarantee for finding an optimal solution, independent of the heuristic.
In conclusion, the cost-to-go heuristic and computation budget allow for finding paths with minimal WM cost within a practical runtime.
\begin{figure}[t]
    \centering

    \includegraphics[width=.45\textwidth]{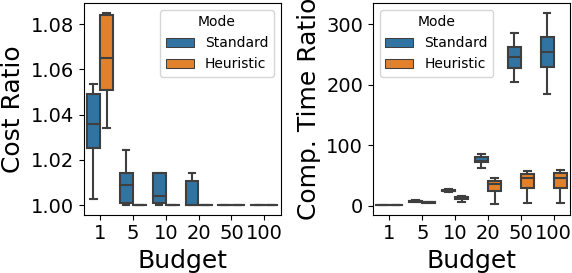}  

    \caption{ 
     Performance of heuristic variants of Algorithm \ref{alg:graph search} with different computation budgets $b$.  Shown are the cost ratio over an optimal solution and the computation time ratio over Dijkstra's algorithm.
     }
    \label{fig:eval_alg_performance}
\end{figure}
%
%
\section{Discussion}
We studied WM cost functions as an alternative to commonly used WS costs in motion planning problems with multiple objectives. We showed that while the WS method is widely used, it might only represent a small subset of all optimal trade-offs when at least one of the objectives is non-convex. We proposed a WM approach as an alternative cost function, which is Pareto-complete. Further, we showed how the WM cost can be used in continuous-space planning, characterized the hardness for graph-based planning and presented a novel path planning algorithm. Our simulations showed that the proposed WM cost is substantially more expressive than the WS across different motion planning problems, and that our proposed path planning algorithm can efficiently find close-to-optimal solutions.
{While the WM formulation makes path planning on graphs NP-hard, our simulation results show that it allows for finding substantially richer sets of solutions, recovering all parts of the Pareto-front. Further, using runtime budgeting and the cost-to-go heuristic allows for computing close-to-optimal solutions within a practical computation time.}

Future work should consider how WM costs can used for learning user preferences in human-in-the-loop learning problems. Given its advantageous expressiveness the WM allows for designing user models that represent a wider variety of user preferences with the same number of parameters.
Another research direction is investigating how robot complex multi-robot routing problems can be solved for a WM cost. 
{Further, for discrete space planning we assumed a monotonic composition function. Future work could include non-monotonic cost functions to broaden the range of applications.}
Lastly, finding suitable parameters for the WM cost remains a challenge. Thus, we plan to adapt our earlier work \cite{botros2022error} to find sets of representative weights for the WM cost.
\bibliographystyle{IEEEtran}

\end{document}